\documentclass{article}

\usepackage{titling}

\usepackage{times}
\usepackage{graphicx} 
\usepackage[lofdepth,lotdepth]{subfig}

\usepackage{natbib}

\usepackage{algpseudocode}
\usepackage{algorithm}

\usepackage{amssymb}
\usepackage{amsthm}
\usepackage{amsmath}

\newtheorem{thm}{Theorem}
\newtheorem*{thm*}{Theorem}
\newtheorem{defn}{Definition}
\newtheorem{lem}{Lemma}
\newtheorem*{lem*}{Lemma}

\newtheorem*{cor*}{Corollary}

\newtheorem{innercustomthm}{Theorem}
\newenvironment{customthm}[1]
  {\renewcommand\theinnercustomthm{#1}\innercustomthm}
  {\endinnercustomthm}
\newtheorem{innercustomlem}{Lemma}
\newenvironment{customlem}[1]
  {\renewcommand\theinnercustomlem{#1}\innercustomlem}
  {\endinnercustomlem}

\usepackage{booktabs}
\usepackage{array}

\usepackage{hyperref}

\usepackage[accepted]{icml2017}

\icmltitlerunning{Fast $k$-Nearest Neighbour Search via Prioritized DCI}

\begin{document} 

\twocolumn[
\icmltitle{Fast $k$-Nearest Neighbour Search via Prioritized DCI}



\icmlsetsymbol{equal}{*}

\begin{icmlauthorlist}
\icmlauthor{Ke Li}{berkeley}
\icmlauthor{Jitendra Malik}{berkeley}
\end{icmlauthorlist}

\icmlaffiliation{berkeley}{University of California, Berkeley, CA 94720, United States}

\icmlcorrespondingauthor{Ke Li}{ke.li@eecs.berkeley.edu}

\icmlkeywords{dynamic continuous indexing, nearest neighbour search, k-nearest neighbours, maximum inner product search, maximum cosine similarity search, randomized algorithm, sublinear time algorithm, spatial data structure, locality-sensitive hashing, k-d tree, random projection, space partitioning, curse of dimensionality}

\vskip 0.3in
]



\printAffiliationsAndNotice{}  

\begin{abstract}
Most exact methods for $k$-nearest neighbour search suffer from the curse of dimensionality; that is, their query times exhibit exponential dependence on either the ambient or the intrinsic dimensionality. Dynamic Continuous Indexing (DCI)~\cite{Li2016FastKN} offers a promising way of circumventing the curse and successfully reduces the dependence of query time on intrinsic dimensionality from exponential to sublinear. In this paper, we propose a variant of DCI, which we call Prioritized DCI, and show a remarkable improvement in the dependence of query time on intrinsic dimensionality. In particular, a \emph{linear} increase in intrinsic dimensionality, or equivalently, an \emph{exponential} increase in the number of points near a query, can be mostly counteracted with just a \emph{linear} increase in space. We also demonstrate empirically that Prioritized DCI significantly outperforms prior methods. In particular, relative to Locality-Sensitive Hashing (LSH), Prioritized DCI reduces the number of distance evaluations by a factor of 14 to 116 and the memory consumption by a factor of 21. 
\end{abstract}

\section{Introduction}

The method of $k$-nearest neighbours is a fundamental building block of many machine learning algorithms and also has broad applications beyond artificial intelligence, including in statistics, bioinformatics and database systems, e.g. \cite{biau2011weighted,behnam2013geometric,eldawy2015spatialhadoop}. Consequently, since the problem of nearest neighbour search was first posed by \citet{minsky1969perceptrons}, it has for decades intrigued the artificial intelligence and theoretical computer science communities alike. Unfortunately, the myriad efforts at devising efficient algorithms have encountered a recurring obstacle: \emph{the curse of dimensionality}, which describes the phenomenon of query time complexity depending exponentially on dimensionality. As a result, even on datasets with moderately high dimensionality, practitioners often have resort to na\"{i}ve exhaustive search. 

Two notions of dimensionality are commonly considered. The more familiar notion, ambient dimensionality, refers to the dimensionality of the space data points are embedded in. On the other hand, intrinsic dimensionality\footnote{The measure of intrinsic dimensionality used throughout this paper is the expansion dimension, also known as the KR-dimension, which is defined as $\log_{2}c$, where $c$ is the expansion rate introduced in \cite{karger2002finding}.} characterizes the intrinsic properties of the data and measures the rate at which the number of points inside a ball grows as a function of its radius. More precisely, for a dataset with intrinsic dimension $d'$, any ball of radius $r$ contains at most $O(r^{d'})$ points. Intuitively, if the data points are uniformly distributed on a manifold, then the intrinsic dimensionality is roughly the dimensionality of the manifold. 

Most existing methods suffer from some form of curse of dimensionality. Early methods like $k$-d trees~\cite{bentley1975multidimensional} and R-trees~\cite{guttman1984r} have query times that grow exponentially in ambient dimensionality. Later methods~\cite{krauthgamer2004navigating,beygelzimer2006cover,dasgupta2008random} overcame the exponential dependence on ambient dimensionality, but have not been able to escape from an exponential dependence on intrinsic dimensionality. Indeed, since a linear increase in the intrinsic dimensionality results in an exponential increase in the number of points near a query, the problem seems fundamentally hard when intrinsic dimensionality is high. 

Recently, \citet{Li2016FastKN} proposed an approach known as Dynamic Continuous Indexing (DCI) that successfully reduces the dependence on intrinsic dimensionality from exponential to sublinear, thereby making high-dimensional nearest neighbour search more practical. The key observation is that the difficulties encountered by many existing methods, including $k$-d trees and Locality-Sensitive Hashing (LSH)~\cite{indyk1998approximate}, may arise from their reliance on space partitioning, which is a popular divide-and-conquer strategy. It works by partitioning the vector space into discrete cells and maintaining a data structure that keeps track of the points lying in each cell. At query time, these methods simply look up of the contents of the cell containing the query and possibly adjacent cells and perform brute-force search over points lying in these cells. While this works well in low-dimensional settings, would it work in high dimensions?

Several limitations of this approach in high-dimensional space are identified in \citep{Li2016FastKN}. First, because the volume of space grows exponentially in dimensionality, either the number or the volumes of cells must grow exponentially. Second, the discretization of the space essentially limits the ``field of view'' of the algorithm, as it is unaware of points that lie in adjacent cells. This is especially problematic when the query lies near a cell boundary, as there could be points in adjacent cells that are much closer to the query. Third, as dimensionality increases, surface area grows faster than volume; as a result, points are increasingly likely to lie near cell boundaries. Fourth, when the dataset exhibits varying density across space, choosing a good partitioning is non-trivial. Furthermore, once chosen, the partitioning is fixed and cannot adapt to changes in density arising from updates to the dataset. 

In light of these observations, DCI is built on the idea of avoiding partitioning the vector space. Instead, it constructs a number of indices, each of which imposes an ordering of all data points. Each index is constructed so that two points with similar ranks in the associated ordering are nearby along a certain random direction. These indices are then combined to allow for retrieval of points that are close to the query along multiple random directions. 

In this paper, we propose a variant of DCI, which assigns a priority to each index that is used to determine which index to process in the upcoming iteration. For this reason, we will refer to this algorithm as Prioritized DCI. This simple change results in a significant improvement in the dependence of query time on intrinsic dimensionality. Specifically, we show a remarkable result: a \emph{linear} increase in intrinsic dimensionality, which could mean an \emph{exponential} increase in the number of points near a query, can be mostly counteracted with a corresponding \emph{linear} increase in the number of indices. In other words, Prioritized DCI can make a dataset with high intrinsic dimensionality seem almost as easy as a dataset with low intrinsic dimensionality, with just a linear increase in space. To our knowledge, there had been no exact method that can cope with high intrinsic dimensionality; Prioritized DCI represents the first method that can do so. 

We also demonstrate empirically that Prioritized DCI significantly outperforms prior methods. In particular, compared to LSH, it achieves a 14- to 116-fold reduction in the number of distance evaluations and a 21-fold reduction in the memory usage. 

\section{Related Work}

There is a vast literature on algorithms for nearest neighbour search. They can be divided into two categories: exact algorithms and approximate algorithms. Early exact algorithms are deterministic and store points in tree-based data structures. Examples include $k$-d trees~\cite{bentley1975multidimensional}, R-trees~\cite{guttman1984r} and X-trees~\cite{berchtold1996x,berchtold1998fast}, which divide the vector space into a hierarchy of half-spaces, hyper-rectangles or Voronoi polygons and keep track of the points that lie in each cell. While their query times are logarithmic in the size of the dataset, they exhibit exponential dependence on the ambient dimensionality. A different method~\cite{meiser1993point} partitions the space by intersecting multiple hyperplanes. It effectively trades off space for time and achieves polynomial query time in ambient dimensionality at the cost of exponential space complexity in ambient dimensionality. 

\begin{figure}[h]
    \centering
    \includegraphics[width=0.9\columnwidth]{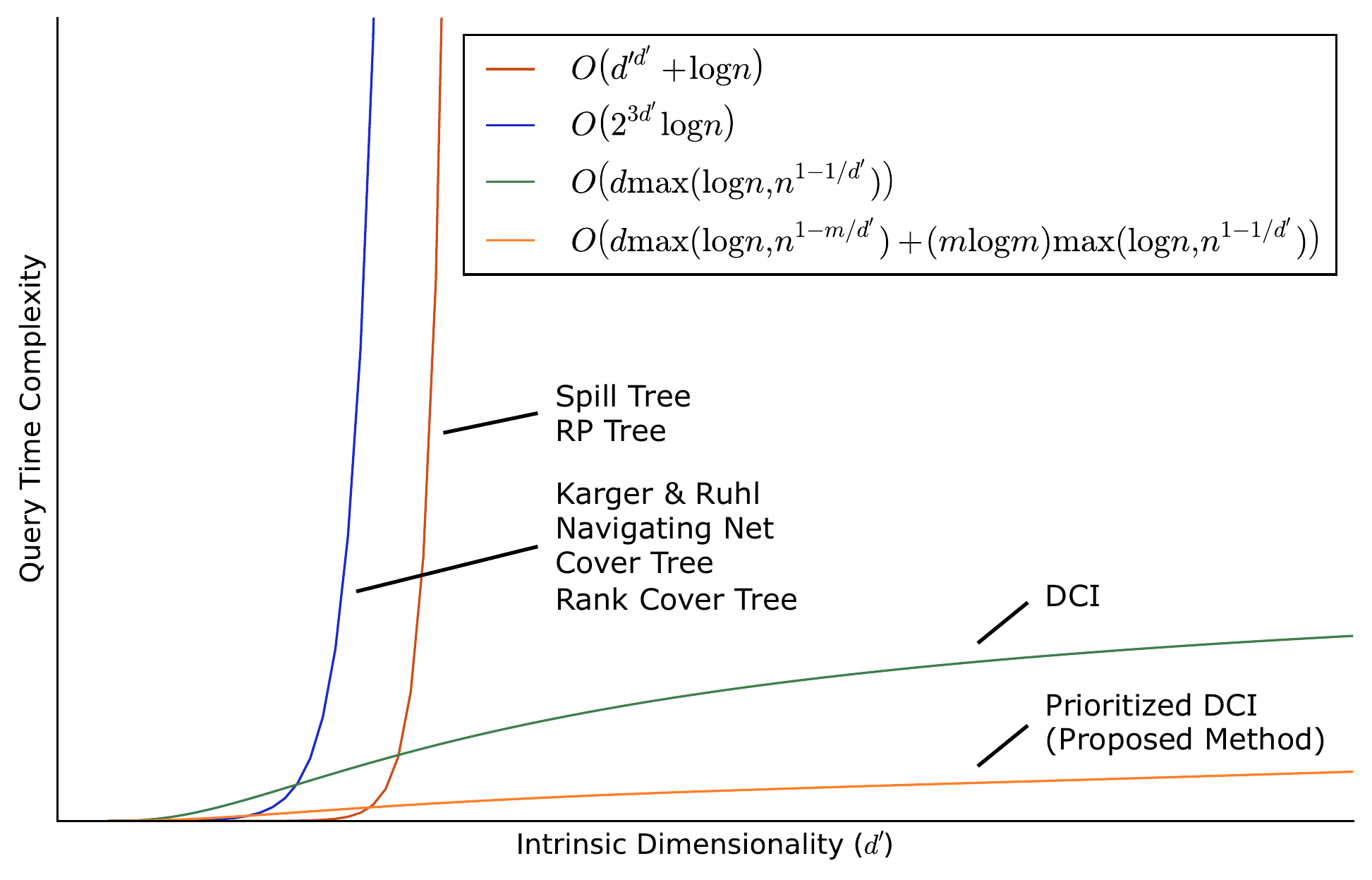}
    \caption{\label{fig:query_time_growth} Visualization of the query time complexities of various exact algorithms as a function of the intrinsic dimensionality $d'$. Each curve represents an example from a class of similar query time complexities. Algorithms that fall into each particular class are shown next to the corresponding curve.}
\end{figure}

To avoid poor performance on worst-case configurations of the data, exact randomized algorithms have been proposed. Spill trees~\cite{liu2004investigation}, RP trees~\cite{dasgupta2008random} and virtual spill trees~\cite{dasgupta2015randomized} extend the ideas behind $k$-d trees by randomizing the orientations of hyperplanes that partition the space into half-spaces at each node of the tree. While randomization enables them to avoid exponential dependence on the ambient dimensionality, their query times still scale exponentially in the intrinsic dimensionality. Whereas these methods rely on space partitioning, other algorithms~\cite{orchard1991fast,clarkson1999nearest,karger2002finding} have been proposed that utilize local search strategies. These methods start with a random point and look in the neighbourhood of the current point to find a new point that is closer to the query than the original in each iteration. Like space partitioning-based approaches, the query time of \cite{karger2002finding} scales exponentially in the intrinsic dimensionality. While the query times of \cite{orchard1991fast,clarkson1999nearest} do not exhibit such undesirable dependence, their space complexities are quadratic in the size of the dataset, making them impractical for large datasets. A different class of algorithms performs search in a coarse-to-fine manner. Examples include navigating nets~\cite{krauthgamer2004navigating}, cover trees~\cite{beygelzimer2006cover} and rank cover trees~\cite{houle2015rank}, which maintain sets of subsampled data points at different levels of granularity and descend through the hierarchy of neighbourhoods of decreasing radii around the query. Unfortunately, the query times of these methods again scale exponentially in the intrinsic dimensionality. 

\begin{table}
\centering
\footnotesize
\begin{tabular}{ll}
\toprule 
Method & Query Time Complexity\\
\midrule
\multicolumn{2}{l}{\emph{Exact Algorithms:}}\\
  RP Tree & $O((d'\log d')^{d'} + \log n)$\\
  Spill Tree & $O(d'^{d'} + \log n)$\\
  \citet{karger2002finding} & $O(2^{3d'}\log n)$\\
  Navigating Net & $2^{O(d')}\log n$\\
  Cover Tree & $O(2^{12d'}\log n)$\\
  Rank Cover Tree & $O(2^{O(d'\log h)}n^{2/h})$ for $h \geq 3$\\
  DCI & $O(d\max(\log n, n^{1 - 1/ d'}))$\\
  Prioritized DCI & $O(d\max(\log n, n^{1 - m/d'})$ \\
  \,(Proposed Method) & \quad $ + (m \log m)\max(\log n, n^{1 - 1/d'}))$ \\
  & \quad for $m \geq 1$\\
\midrule
\multicolumn{2}{l}{\emph{Approximate Algorithms:}}\\
  $k$-d Tree & $O((1/\epsilon)^{d}\log n)$\\
  BBD Tree & $O((6/\epsilon)^{d}\log n)$\\
  LSH & $\approx O(dn^{1/(1+\epsilon)^{2}})$\\
\bottomrule
\end{tabular}
\caption{Query time complexities of various algorithms for 1-NN search. Ambient dimensionality, intrinsic dimensionality, dataset size and approximation ratio are denoted as $d$, $d'$, $n$ and $1+\epsilon$. A visualization of the growth of various time complexities as a function of the intrinsic dimensionality is shown in Figure~\ref{fig:query_time_growth}. }
\label{tab:query_time_comparison}
\end{table}

Due to the difficulties of devising efficient algorithms for the exact version of the problem, there has been extensive work on approximate algorithms. Under the approximate setting, returning any point whose distance to the query is within a factor of $1 + \epsilon$ of the distance between the query and the true nearest neighbour is acceptable. Many of the same strategies are employed by approximate algorithms. Methods based on tree-based space partitioning~\cite{arya1998optimal} and local search~\cite{arya1993approximate} have been developed; like many exact algorithms, their query times also scale exponentially in the ambient dimensionality. Locality-Sensitive Hashing (LSH)~\cite{indyk1998approximate,datar2004locality,andoni2006near} partitions the space into regular cells, whose shapes are implicitly defined by the choice of the hash function. It achieves a query time of $O(dn^{\rho})$ using $O(dn^{1+\rho})$ space, where $d$ is the ambient dimensionality, $n$ is the dataset size and $\rho \approx 1 / (1 + \epsilon)^{2}$ for large $n$ in Euclidean space, though the dependence on intrinsic dimensionality is not made explicit. In practice, the performance of LSH degrades on datasets with large variations in density, due to the uneven distribution of points across cells. Consequently, various data-dependent hashing schemes have been proposed~\cite{pauleve2010locality,weiss2009spectral,andoni2015optimal}; unlike data-independent hashing schemes, however, they do not allow dynamic updates to the dataset. A related approach~\cite{jegou2011product} decomposes the space into mutually orthogonal axis-aligned subspaces and independently partitions each subspace. It has a query time linear in the dataset size and no known guarantee on the probability of correctness under the exact or approximate setting. A different approach~\cite{anagnostopoulos2014low} projects the data to a lower dimensional space that approximately preserves approximate nearest neighbour relationships and applies other approximate algorithms like BBD trees~\cite{arya1998optimal} to the projected data. Its query time is also linear in ambient dimensionality and sublinear in the dataset size. Unlike LSH, it uses space linear in the dataset size, at the cost of longer query time than LSH. Unfortunately, its query time is exponential in intrinsic dimensionality. 

Our work is most closely related to Dynamic Continuous Indexing (DCI)~\cite{Li2016FastKN}, which is an exact randomized algorithm for Euclidean space whose query time is linear in ambient dimensionality, sublinear in dataset size and sublinear in intrinsic dimensionality and uses space linear in the dataset size. Rather than partitioning the vector space, it uses multiple global one-dimensional indices, each of which orders data points along a certain random direction and combines these indices to find points that are near the query along multiple random directions. The proposed algorithm builds on the ideas introduced by DCI and achieves a significant improvement in the dependence on intrinsic dimensionality. 

A summary of the query times of various prior algorithms and the proposed algorithm is presented in Table~\ref{tab:query_time_comparison} and their growth as a function of intrinsic dimensionality is illustrated in Figure~\ref{fig:query_time_growth}. 

\section{Prioritized DCI}

DCI constructs a data structure consisting of multiple \emph{composite indices} of data points, each of which in turn consists of a number of \emph{simple indices}. Each simple index orders data points according to their projections along a particular random direction. Given a query, for every composite index, the algorithm finds points that are near the query in every constituent simple index, which are known as \emph{candidate points}, and adds them to a set known as the \emph{candidate set}. The true distances from the query to every candidate point are evaluated and the ones that are among the $k$ closest to the query are returned. 

More concretely, each simple index is associated with a random direction and stores the projections of every data point along the direction. They are implemented using standard data structures that maintain one-dimensional ordered sequences of elements, like self-balancing binary search trees~\cite{bayer1972symmetric,guibas1978dichromatic} or skip lists~\cite{pugh1990skip}. At query time, the algorithm projects the query along the projection directions associated with each simple index and finds the position where the query would have been inserted in each simple index, which takes logarithmic time. It then iterates over, or \emph{visits}, data points in each simple index in the order of their distances to the query under projection, which takes constant time for each iteration. As it iterates, it keeps track of how many times each data point has been visited across all simple indices of each composite index. If a data point has been visited in every constituent simple index, it is added to the candidate set and is said to have been \emph{retrieved} from the composite index. 

\begin{algorithm}
\footnotesize
\caption{Data structure construction procedure}
\label{alg_construct}
\begin{algorithmic}
\Require A dataset $D$ of $n$ points $p^{1},\ldots,p^{n}$, the number of simple indices $m$ that constitute a composite index and the number of composite indices $L$
\Function{Construct}{$D,m,L$}
    \State $\{u_{jl}\}_{j \in [m], l \in [L]} \gets mL$ random unit vectors in $\mathbb{R}^{d}$
    \State $\{T_{jl}\}_{j \in [m], l \in [L]} \gets mL$ empty binary search trees or skip 
    \State \;\qquad\qquad\qquad\qquad lists
    \For{$j = 1$ \textbf{to} $m$}
        \For{$l = 1$ \textbf{to} $L$}
            \For{$i = 1$ \textbf{to} $n$}
                \State $\overline{p}^{i}_{jl} \gets \langle p^{i},u_{jl}\rangle$
                \State Insert $(\overline{p}^{i}_{jl},i)$ into $T_{jl}$ with $\overline{p}^{i}_{jl}$ being the key and 
                \State \; $i$ being the value
            \EndFor
        \EndFor
    \EndFor
    \State \Return $\{(T_{jl},u_{jl})\}_{j \in [m], l \in [L]}$
\EndFunction
\end{algorithmic}
\normalsize
\end{algorithm}

DCI has a number of appealing properties compared to methods based on space partitioning. Because points are visited by rank rather than location in space, DCI performs well on datasets with large variations in data density. It naturally skips over sparse regions of the space and concentrates more on dense regions of the space. Since construction of the data structure does not depend on the dataset, the algorithm supports dynamic updates to the dataset, while being able to automatically adapt to changes in data density. Furthermore, because data points are represented in the indices as continuous values without being discretized, the granularity of discretization does not need to be chosen at construction time. Consequently, the same data structure can support queries at varying desired levels of accuracy, which allows a different speed-vs-accuracy trade-off to be made for each individual query. 

Prioritized DCI differs from standard DCI in the order in which points from different simple indices are visited. In standard DCI, the algorithm cycles through all constituent simple indices of a composite index at regular intervals and visits exactly one point from each simple index in each pass. In Prioritized DCI, the algorithm assigns a priority to each constituent simple index; in each iteration, it visits the upcoming point from the simple index with the highest priority and updates the priority at the end of the iteration. The priority of a simple index is set to the negative absolute difference between the query projection and the next data point projection in the index. 

\begin{algorithm}
\footnotesize
\caption{$k$-nearest neighbour querying procedure}
\label{alg_query}
\begin{algorithmic}
\Require Query point $q$ in $\mathbb{R}^{d}$, binary search trees/skip lists and their associated projection vectors $\{(T_{jl},u_{jl})\}_{j \in [m], l \in [L]}$, the number of points to retrieve $k_{0}$ and the number of points to visit $k_{1}$ in each composite index
\Function{Query}{$q,\{(T_{jl},u_{jl})\}_{j,l},k_{0},k_{1}$}
    \State $C_{l} \gets$ array of size $n$ with entries initialized to 0\; $\forall l \in [L]$
    \State $\overline{q}_{jl} \gets \langle q,u_{jl}\rangle\; \forall j \in [m], l \in [L]$
    \State $S_{l} \gets \emptyset \; \forall l \in [L]$
    \State $P_{l} \gets$ empty priority queue $\; \forall l \in [L]$
    \For{$l = 1$ \textbf{to} $L$}
            \For{$j = 1$ \textbf{to} $m$}
                \State $(\overline{p}_{jl}^{(1)},h_{jl}^{(1)}) \gets $ the node in $T_{jl}$ whose key is the 
                \State \qquad\qquad\qquad closest to $\overline{q}_{jl}$
                \State Insert $(\overline{p}_{jl}^{(1)},h_{jl}^{(1)}) $ with priority $-|\overline{p}_{jl}^{(1)} - \overline{q}_{jl}|$ 
                \State \quad into $P_{l}$
            \EndFor
    \EndFor
        \For{$i' = 1$ \textbf{to} $k_{1} - 1$}
            \For{$l = 1$ \textbf{to} $L$}
                \If{$\left|S_{l}\right| < k_{0}$}
                    \State $(\overline{p}_{jl}^{(i)},h_{jl}^{(i)}) \gets $ the node with the highest priority
                    \State \qquad\qquad\qquad in $P_{l}$
                    \State Remove $(\overline{p}_{jl}^{(i)},h_{jl}^{(i)})$ from $P_{l}$ and insert the node
                    \State \quad in $T_{jl}$ whose key is the next closest to $\overline{q}_{jl}$, 
                    \State \quad which is denoted as $(\overline{p}_{jl}^{(i+1)},h_{jl}^{(i+1)})$, with
                    \State \quad priority $-|\overline{p}_{jl}^{(i+1)} - \overline{q}_{jl}|$ into $P_{l}$
                    \State $C_{l}[h_{jl}^{(i)}] \gets C_{l}[h_{jl}^{(i)}] + 1$
                    \If{$C_{l}[h_{jl}^{(i)}] = m$}
                        \State $S_{l} \gets S_{l} \cup \{h_{jl}^{(i)}\}$
                    \EndIf
                \EndIf
            \EndFor
        \EndFor
    \State \Return $k$ points in $\bigcup_{l\in[L]}S_{l}$ that are the closest in 
    \State \qquad\;\;\; Euclidean distance in $\mathbb{R}^{d}$ to $q$
\EndFunction
\end{algorithmic}
\normalsize
\end{algorithm}

Intuitively, this ensures data points are visited in the order of their distances to the query under projection. Because data points are only retrieved from a composite index when they have been visited in all constituent simple indices, data points are retrieved in the order of the maximum of their distances to the query along multiple projection directions. Since distance under projection forms a lower bound on the true distance, the maximum projected distance approaches the true distance as the number of projection directions increases. Hence, in the limit as the number of simple indices approaches infinity, data points are retrieved in the ideal order, that is, the order of their true distances to the query. 

The construction and querying procedures of Prioritized DCI are presented formally in Algorithms~\ref{alg_construct} and \ref{alg_query}. To ensure the algorithm retrieves the exact $k$-nearest neighbours with high probability, the analysis in the next section shows that one should choose $k_{0} \in \Omega(k\max(\log(n/k),(n/k)^{1-m/d'}))$ and $k_{1} \in \Omega(mk\max(\log(n/k),(n/k)^{1-1/d'}))$, where $d'$ denotes the intrinsic dimensionality. Though because this assumes worst-case configuration of data points, it may be overly conservative in practice; so, these parameters may be chosen by cross-validation. 

We summarize the time and space complexities of Prioritized DCI in Table \ref{tab:analysis_results}. Notably, the first term of the query complexity, which dominates when the ambient dimensionality $d$ is large, has a more favourable dependence on the intrinsic dimensionality $d'$ than the query complexity of standard DCI. In particular, a linear increase in the intrinsic dimensionality, which corresponds to an exponential increase in the expansion rate, can be mitigated by just a linear increase in the number of simple indices $m$. This suggests that Prioritized DCI can better handle datasets with high intrinsic dimensionality than standard DCI, which is confirmed by empirical evidence later in this paper. 

\section{Analysis}

We analyze the time and space complexities of Prioritized DCI below and derive the stopping condition of the algorithm. Because the algorithm uses standard data structures, analysis of the construction time, insertion time, deletion time and space complexity is straightforward. Hence, this section focuses mostly on analyzing the query time. 

\begin{table}
\centering
\footnotesize
\begin{tabular}{ll}
\toprule 
Property & Complexity\\
\midrule
Construction & $O(m(dn+n\log n))$\\
Query & $O\left(dk\max(\log(n/k),(n/k)^{1-m/d'})+\right.$\\
 & $\left.mk\log m\left(\max(\log(n/k),(n/k)^{1-1/d'})\right)\right)$\\
Insertion & $O(m(d+\log n))$\\
Deletion & $O(m\log n)$\\
Space & $O(mn)$\\
\bottomrule
\end{tabular}
\caption{Time and space complexities of Prioritized DCI. }
\label{tab:analysis_results}
\end{table}

In high-dimensional space, query time is dominated by the time spent on evaluating true distances between candidate points and the query. Therefore, we need to find the number of candidate points that must be retrieved to ensure the algorithm succeeds with high probability. To this end, we derive an upper bound on the failure probability for any given number of candidate points. The algorithm fails if sufficiently many distant points are retrieved from each composite index before some of the true $k$-nearest neighbours. We decompose this event into multiple (dependent) events, each of which is the event that a particular distant point is retrieved before some true $k$-nearest neighbours. Since points are retrieved in the order of their maximum projected distance, this event happens when the maximum projected distance of the distant point is less than that of a true $k$-nearest neighbour. We start by finding an upper bound on the probability of this event. To simplify notation, we initially consider displacement vectors from the query to each data point, and so relationships between projected distances of triplets of points translate relationships between projected lengths of pairs of displacement vectors. 

We start by examining the event that a vector under random one-dimensional projection satisfies some geometric constraint. We then find an upper bound on the probability that some combinations of these events occur, which is related to the failure probability of the algorithm. 

\begin{lem}
\label{lem:order_invert}
Let $v^{l},v^{s} \in \mathbb{R}^{d}$ be such that $\left\Vert v^{l} \right\Vert _{2} > \left\Vert v^{s}\vphantom{v^{l}} \right\Vert _{2}$, $\left\{ u_{j}'\right\} _{j=1}^{M}$ be i.i.d. unit vectors in $\mathbb{R}^{d}$ drawn uniformly at random. Then $\mathrm{Pr}\left(\max_{j}\left\{ \left|\langle v^{l},u_{j}'\rangle\right|\right\} \leq\left\Vert v^{s}\vphantom{v^{l}}\right\Vert _{2}\right) = \left(1-\frac{2}{\pi}\cos^{-1}\left(\left\Vert v^{s}\vphantom{v^{l}}\right\Vert _{2}/\left\Vert v^{l}\right\Vert _{2}\right)\right)^{M}$. 
\end{lem}
\begin{proof}
The event $\left\{ \max_{j}\left\{ \left|\langle v^{l},u_{j}'\rangle\right|\right\} \leq\left\Vert v^{s}\right\Vert _{2}\right\}$ is equivalent to the event that $\left\{ \left|\langle v^{l},u_{j}'\rangle\right|\leq\left\Vert v^{s}\right\Vert _{2}\;\forall j\right\}$, which is the intersection of the events $\left\{ \left|\langle v^{l},u_{j}'\rangle\right|\leq\left\Vert v^{s}\right\Vert _{2}\right\}$. Because $u_{j}'$'s are drawn independently, these events are independent. 

Let $\theta_{j}$ be the angle between $v^{l}$ and $u_{j}'$, so that $\langle v^{l},u_{j}'\rangle=\left\Vert v^{l}\right\Vert _{2}\cos\theta_{j}$. Since $u_{j}'$ is drawn uniformly, $\theta_{j}$ is uniformly distributed on $[0,2\pi]$. Hence,
\vspace{-5pt}
\footnotesize
\begin{align*}
\; & \mathrm{Pr}\left(\max_{j}\left\{ \left|\langle v^{l},u_{j}'\rangle\right|\right\} \leq\left\Vert v^{s}\right\Vert _{2}\right)\\
\; = \;& \prod_{j=1}^{M}\mathrm{Pr}\left(\left|\langle v^{l},u_{j}'\rangle\right|\leq\left\Vert v^{s}\right\Vert _{2}\right)\\
\; = \;& \prod_{j=1}^{M}\mathrm{Pr}\left(\left|\cos\theta_{j}\right|\leq\frac{\left\Vert v^{s}\right\Vert _{2}}{\left\Vert v^{l}\right\Vert _{2}}\right) \\  = \;& \prod_{j=1}^{M} \left( 2\mathrm{Pr}\left(\theta_{j}\in\left[\cos^{-1}\left(\frac{\left\Vert v^{s}\right\Vert _{2}}{\left\Vert v^{l}\right\Vert _{2}}\right),\pi-\cos^{-1}\left(\frac{\left\Vert v^{s}\right\Vert _{2}}{\left\Vert v^{l}\right\Vert _{2}}\right)\right]\right) \right)\\
\; = \;& \left(1-\frac{2}{\pi}\cos^{-1}\left(\frac{\left\Vert v^{s}\right\Vert _{2}}{\left\Vert v^{l}\right\Vert _{2}}\right)\right)^{M}
\end{align*}
\normalsize
\end{proof}

\begin{lem}
\label{lem:choose_event}
For any set of events $\left\{ E_{i}\right\} _{i=1}^{N}$, the probability that at least $k'$ of them happen is at most $\frac{1}{k'}\sum_{i=1}^{N}\mathrm{Pr}\left(E_{i}\right)$. 
\end{lem}
\begin{proof}
For any set $T\subseteq[N]$, define $\tilde{E}_{T}$ to be the intersection of events indexed by $T$ and complements of events not indexed by $T$, i.e. $\tilde{E}_{T}=\left(\bigcap_{i\in T}E_{i}\right)\cap\left(\bigcap_{i\notin T}\overline{E}_{i}\right)$. Observe that $\left\{ \tilde{E}_{T}\right\} _{T\subseteq[N]}$ are disjoint and that for any $I\subseteq[N]$, $\bigcap_{i\in I}E_{i}=\bigcup_{T\supseteq I}\tilde{E}_{T}$. The event that at least $k'$ of $E_{i}$'s happen is $\bigcup_{I\subseteq[N]:|I|=k'}\bigcap_{i\in I}E_{i}$, which is equivalent to $\bigcup_{I\subseteq[N]:|I|=k'}\bigcup_{T\supseteq I}\tilde{E}_{T}=\bigcup_{T\subseteq[N]:|T|\geq k'}\tilde{E}_{T}$. We will henceforth use $\mathcal{T}$ to denote $\left\{ T\subseteq[N]:|T|\geq k'\right\}$. Since $\mathcal{T}$ is a finite set, we can impose an ordering on its elements and denote the $l^{\mathrm{th}}$ element as $T_{l}$. The event can therefore be rewritten as $\bigcup_{l=1}^{\left|\mathcal{T}\right|}\tilde{E}_{T_{l}}$. 

Define $E'_{i,j}$ to be $E_{i}\setminus\left(\bigcup_{l=j+1}^{\left|\mathcal{T}\right|}\tilde{E}_{T_{l}}\right)$. We claim that $\sum_{i=1}^{N}\mathrm{Pr}\left(E'_{i,j}\right)\geq k'\sum_{l=1}^{j}\mathrm{Pr}\left(\tilde{E}_{T_{l}}\right)$ for all $j\in\{0,\ldots,\left|\mathcal{T}\right|\}$. We will show this by induction on $j$. 
 
For $j=0$, the claim is vacuously true because probabilities are non-negative. 
For $j>0$, we observe that $E'_{i,j}=\left(E'_{i,j}\setminus\tilde{E}_{T_{j}}\right)\cup\left(E'_{i,j}\cap\tilde{E}_{T_{j}}\right)=E'_{i,j-1}\cup\left(E'_{i,j}\cap\tilde{E}_{T_{j}}\right)$ for all $i$. Since $E'_{i,j}\setminus\tilde{E}_{T_{j}}$ and $E'_{i,j}\cap\tilde{E}_{T_{j}}$ are disjoint, $\mathrm{Pr}\left(E'_{i,j}\right)=\mathrm{Pr}\left(E'_{i,j-1}\right)+\mathrm{Pr}\left(E'_{i,j}\cap\tilde{E}_{T_{j}}\right)$. 

Consider the quantity $\sum_{i\in T_{j}}\mathrm{Pr}\left(E'_{i,j}\right)$, which is $\sum_{i\in T_{j}}\left(\mathrm{Pr}\left(E'_{i,j-1}\right)+\mathrm{Pr}\left(E'_{i,j}\cap\tilde{E}_{T_{j}}\right)\right)$ by the above observation. For each $i\in T_{j}\ensuremath{,}\tilde{E}_{T_{j}}\subseteq E_{i}$, and so $\tilde{E}_{T_{j}}\setminus\left(\bigcup_{l=j+1}^{\left|\mathcal{T}\right|}\tilde{E}_{T_{l}}\right)\subseteq E_{i}\setminus\left(\bigcup_{l=j+1}^{\left|\mathcal{T}\right|}\tilde{E}_{T_{l}}\right)=E'_{i,j}$. Because $\left\{ \tilde{E}_{T_{l}}\right\} _{l=j}^{\left|\mathcal{T}\right|}$ are disjoint, $\tilde{E}_{T_{j}}\setminus\left(\bigcup_{l=j+1}^{\left|\mathcal{T}\right|}\tilde{E}_{T_{l}}\right)=\tilde{E}_{T_{j}}$. Hence, $\tilde{E}_{T_{j}}\subseteq E'_{i,j}$ and so $E'_{i,j}\cap\tilde{E}_{T_{j}}=\tilde{E}_{T_{j}}$. Thus, $\sum_{i\in T_{j}}\mathrm{Pr}\left(E'_{i,j}\right)=\left|T_{j}\right|\mathrm{Pr}\left(\tilde{E}_{T_{j}}\right)+\sum_{i\in T_{j}}\mathrm{Pr}\left(E'_{i,j-1}\right)$. 

It follows that $\sum_{i=1}^{N}\mathrm{Pr}\left(E'_{i,j}\right)=\left|T_{j}\right|\mathrm{Pr}\left(\tilde{E}_{T_{j}}\right)+\sum_{i\in T_{j}}\mathrm{Pr}\left(E'_{i,j-1}\right)+\sum_{i\notin T_{j}}\mathrm{Pr}\left(E'_{i,j}\right)$. Because $\mathrm{Pr}\left(E'_{i,j}\right)=\mathrm{Pr}\left(E'_{i,j-1}\right)+\mathrm{Pr}\left(E'_{i,j}\cap\tilde{E}_{T_{j}}\right)\geq\mathrm{Pr}\left(E'_{i,j-1}\right)$ and $\left|T_{j}\right|\geq k'$, $\sum_{i=1}^{N}\mathrm{Pr}\left(E'_{i,j}\right)\geq k'\mathrm{Pr}\left(\tilde{E}_{T_{j}}\right)+\sum_{i=1}^{N}\mathrm{Pr}\left(E'_{i,j-1}\right)$. By the inductive hypothesis, $\sum_{i=1}^{N}\mathrm{Pr}\left(E'_{i,j-1}\right)\geq k'\sum_{l=1}^{j-1}\mathrm{Pr}\left(\tilde{E}_{T_{l}}\right)$. Therefore, $\sum_{i=1}^{N}\mathrm{Pr}\left(E'_{i,j}\right)\geq k'\sum_{l=1}^{j}\mathrm{Pr}\left(\tilde{E}_{T_{l}}\right)$, which concludes the induction argument. 

The lemma is a special case of this claim when $j=\left|\mathcal{T}\right|$, since $E'_{i,\left|\mathcal{T}\right|}=E_{i}$ and $\sum_{l=1}^{\left|\mathcal{T}\right|}\mathrm{Pr}\left(\tilde{E}_{T_{l}}\right)=\mathrm{Pr}\left(\bigcup_{l=1}^{\left|\mathcal{T}\right|}\tilde{E}_{T_{l}}\right)$. 

\end{proof}

Combining the above yields the following theorem, the proof of which is found in the supplementary material. 

\begin{thm}
\label{thm:multi_order_invert}
Let $\left\{ v^{l}_{i}\right\} _{i=1}^{N}$ and $\left\{ v^{s}_{i'} \vphantom{v^{l}_{i}} \right\} _{i'=1}^{N'}$ be sets of vectors such that $\left\Vert v^{l}_{i}\right\Vert _{2} > \left\Vert \vphantom{v^{l}} v^{s}_{i'}\right\Vert _{2}\;\forall i\in[N],i'\in[N']$. Furthermore, let $\left\{ u_{ij}'\right\} _{i\in[N],j\in[M]}$ be random uniformly distributed unit vectors such that $u_{i1}',\ldots,u_{iM}'$ are independent for any given $i$. Consider the events $\left\{ \exists v_{i'}^{s}\mbox{ s.t. }\max_{j}\left\{ \left|\langle v_{i}^{l},u_{ij}'\rangle\right|\right\} \leq\left\Vert v_{i'}^{s}\vphantom{v^{l}}\right\Vert _{2}\right\} _{i=1}^{N}$. The probability that at least $k'$ of these events occur is at most $\frac{1}{k'}\sum_{i=1}^{N}\left(1-\frac{2}{\pi}\cos^{-1}\left(\left\Vert v^{s}_{\mathrm{max}}\vphantom{v^{l}}\right\Vert _{2}/\left\Vert v_{i}^{l}\right\Vert _{2}\right)\right)^{M}$, where $\left\Vert v^{s}_{\mathrm{max}}\vphantom{v^{l}}\right\Vert _{2} = \max_{i'}\left\{ \left\Vert v_{i'}^{s}\vphantom{v^{l}}\right\Vert _{2}\right\} $. Furthermore, if $k' = N$, it is at most $\min_{i\in[N]}\left\{ \left(1-\frac{2}{\pi}\cos^{-1}\left(\left\Vert v^{s}_{\mathrm{max}}\vphantom{v^{l}}\right\Vert _{2}/\left\Vert v_{i}^{l}\right\Vert _{2}\right)\right)^{M} \right\}$. 
\end{thm}

We now apply the results above to analyze specific properties of the algorithm. For convenience, instead of working directly with intrinsic dimensionality, we will analyze the query time in terms of a related quantity, global relative sparsity, as defined in \cite{Li2016FastKN}. We reproduce its definition below for completeness. 

\begin{defn}
Given a dataset $D\subseteq\mathbb{R}^{d}$, let $B_{p}(r)$ be the set of points in $D$ that are within a ball of radius $r$ around a point $p$. A dataset $D$ has global relative sparsity of $(\tau,\gamma)$ if for all $r$ and $p \in \mathbb{R}^{d}$ such that $\left|B_{p}(r)\right|\geq\tau$, $\left|B_{p}(\gamma r)\right|\leq2\left|B_{p}(r)\right|$, where $\gamma \geq 1$. 
\end{defn}

Global relative sparsity is related to the expansion rate~\cite{karger2002finding} and intrinsic dimensionality in the following way: a dataset with global relative sparsity of $(\tau,\gamma)$ has $(\tau,2^{(1/\log_{2}\gamma)})$-expansion and intrinsic dimensionality of $1/\log_{2}\gamma$. 

Below we derive two upper bounds on the probability that some of the true $k$-nearest neighbours are missing from the set of candidate points retrieved from a given composite index, which are in expressed in terms of $k_{0}$ and $k_{1}$ respectively. These results inform us how $k_{0}$ and $k_{1}$ should be chosen to ensure the querying procedure returns the correct results with high probability. In the results that follow, we use $\{p^{(i)}\}_{i=1}^{n}$ to denote a re-ordering of the points $\{p^{i}\}_{i=1}^{n}$ so that $p^{(i)}$ is the $i^{\mathrm{th}}$ closest point to the query $q$. Proofs are found in the supplementary material.

\begin{lem}
\label{lem:prob_num_extraneous_retrieved_points}
Consider points in the order they are retrieved from a composite index that consists of $m$ simple indices. The probability that there are at least $n_{0}$ points that are not the true $k$-nearest neighbours but are retrieved before some of them is at most $\frac{1}{n_{0}-k}\sum_{i=2k+1}^{n}\left(1-\frac{2}{\pi}\cos^{-1}\left(\left\Vert p^{(k)}-q\right\Vert _{2}/\left\Vert p^{(i)}-q\right\Vert _{2}\right)\right)^{m}$. 
\end{lem}

\begin{lem}
\label{lem:prob_num_extraneous_visited_points}
Consider point projections in a composite index that consists of $m$ simple indices in the order they are visited. The probability that $n_{0}$ point projections that are not of the true $k$-nearest neighbours are visited before all true $k$-nearest neighbours have been retrieved is at most $\frac{m}{n_{0}-mk}\sum_{i=2k+1}^{n}\left(1-\frac{2}{\pi}\cos^{-1}\left(\left\Vert p^{(k)}-q\right\Vert _{2}/\left\Vert p^{(i)}-q\right\Vert _{2}\right)\right)$. 
\end{lem}

\begin{lem}
\label{lem:sum_dist_ratios}
On a dataset with global relative sparsity $(k,\gamma)$, the quantity $\sum_{i=2k+1}^{n}\left(1-\frac{2}{\pi}\cos^{-1}\left(\left\Vert p^{(k)}-q\right\Vert _{2}/\left\Vert p^{(i)}-q\right\Vert _{2}\right)\right)^{m}$ is at most $O\left(k\max(\log(n/k),(n/k)^{1-m\log_{2}\gamma})\right)$. 
\end{lem}

\begin{lem}
\label{lem:single_k0_failure_prob}
For a dataset with global relative sparsity $(k,\gamma)$ and a given composite index consisting of $m$ simple indices, there is some $k_{0} \in\Omega(k\max(\log(n/k),(n/k)^{1-m\log_{2}\gamma}))$ such that the probability that the candidate points retrieved from the composite index do not include some of the true $k$-nearest neighbours is at most some constant $\alpha_{0}<1$.
\end{lem}

\begin{lem}
\label{lem:single_k1_failure_prob}
For a dataset with global relative sparsity $(k,\gamma)$ and a given composite index consisting of $m$ simple indices, there is some $k_{1} \in\Omega(mk\max(\log(n/k),(n/k)^{1-\log_{2}\gamma}))$ such that the probability that the candidate points retrieved from the composite index do not include some of the true $k$-nearest neighbours is at most some constant $\alpha_{1}<1$.
\end{lem}

\begin{thm}
\label{thm:data_indep_alg_correctness}
For a dataset with global relative sparsity $(k,\gamma)$, for any $\epsilon > 0$, there is some $L$, $k_{0} \in \Omega(k\max(\log(n/k),(n/k)^{1-m\log_{2}\gamma}))$ and $k_{1} \in \Omega(mk\max(\log(n/k),(n/k)^{1-\log_{2}\gamma}))$ such that the algorithm returns the correct set of $k$-nearest neighbours with probability of at least $1 - \epsilon$. 
\end{thm}

Now that we have found a choice of $k_{0}$ and $k_{1}$ that suffices to ensure correctness with high probability, we can derive a bound on the query time that guarantees correctness. We then analyze the time complexity for construction, insertion and deletion and the space complexity. Proofs of the following are found in the supplementary material. 

\begin{thm}
\label{thm:data_indep_alg_query_time_complexity}
For a given number of simple indices $m$, the algorithm takes $O\left(dk\max(\log(n/k),(n/k)^{1-m/d'})+\right.$ $\left.mk\log m\left(\max(\log(n/k),(n/k)^{1-1/d'})\right)\right)$ time to retrieve the $k$-nearest neighbours at query time, where $d'$ denotes the intrinsic dimensionality. 
\end{thm}

\begin{thm}
\label{thm:data_indep_alg_construction_time_complexity}
For a given number of simple indices $m$, the algorithm takes $O(m(dn+n\log n))$ time to preprocess the data points in $D$ at construction time. 
\end{thm}

\begin{thm}
\label{thm:data_indep_alg_update_time_complexity}
The algorithm requires $O(m(d+\log n))$ time to insert a new data point and $O(m \log n)$ time to delete a data point. 
\end{thm}

\begin{thm}
\label{thm:data_indep_alg_space_complexity}
The algorithm requires $O(mn)$ space in addition to the space used to store the data. 
\end{thm}

\section{Experiments}

\begin{figure*}[t]
    \centering
    \subfloat[]{
        \includegraphics[width=0.33\textwidth]{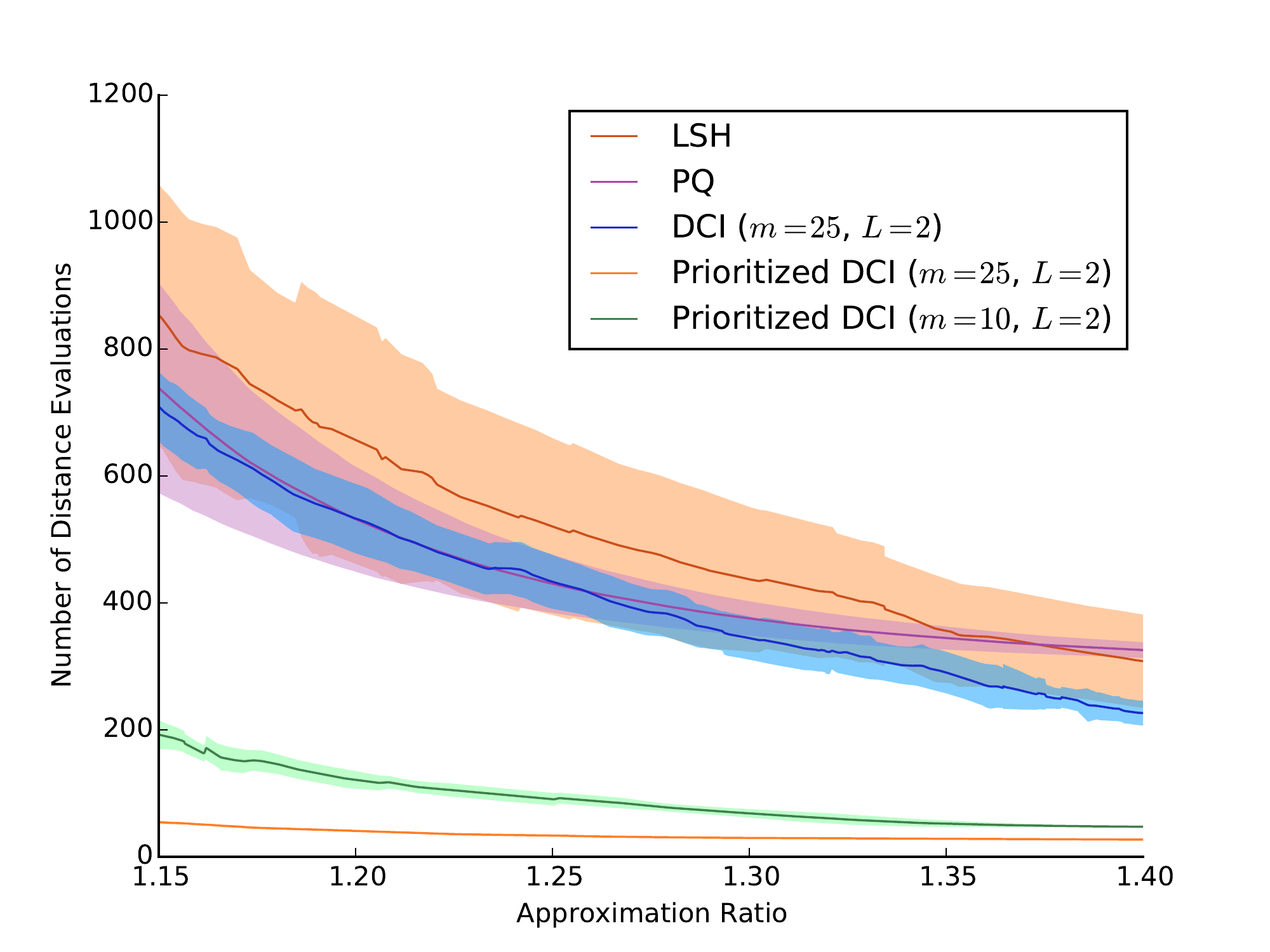}
        \label{fig:dist_evals_cifar}
    }
    \subfloat[]{
        \includegraphics[width=0.33\textwidth]{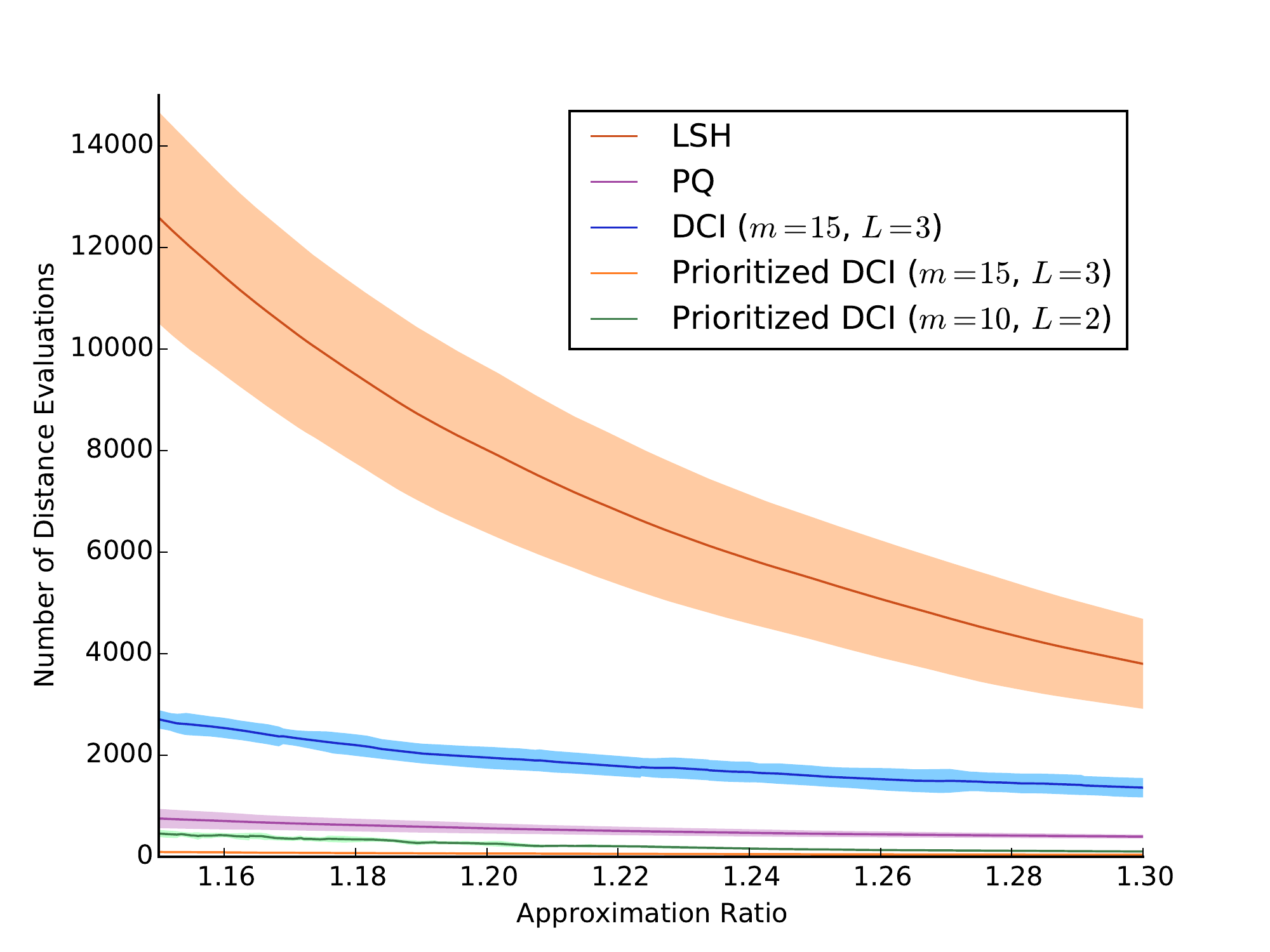}
        \label{fig:dist_evals_mnist}
    }
    \subfloat[]{
        \includegraphics[width=0.33\textwidth]{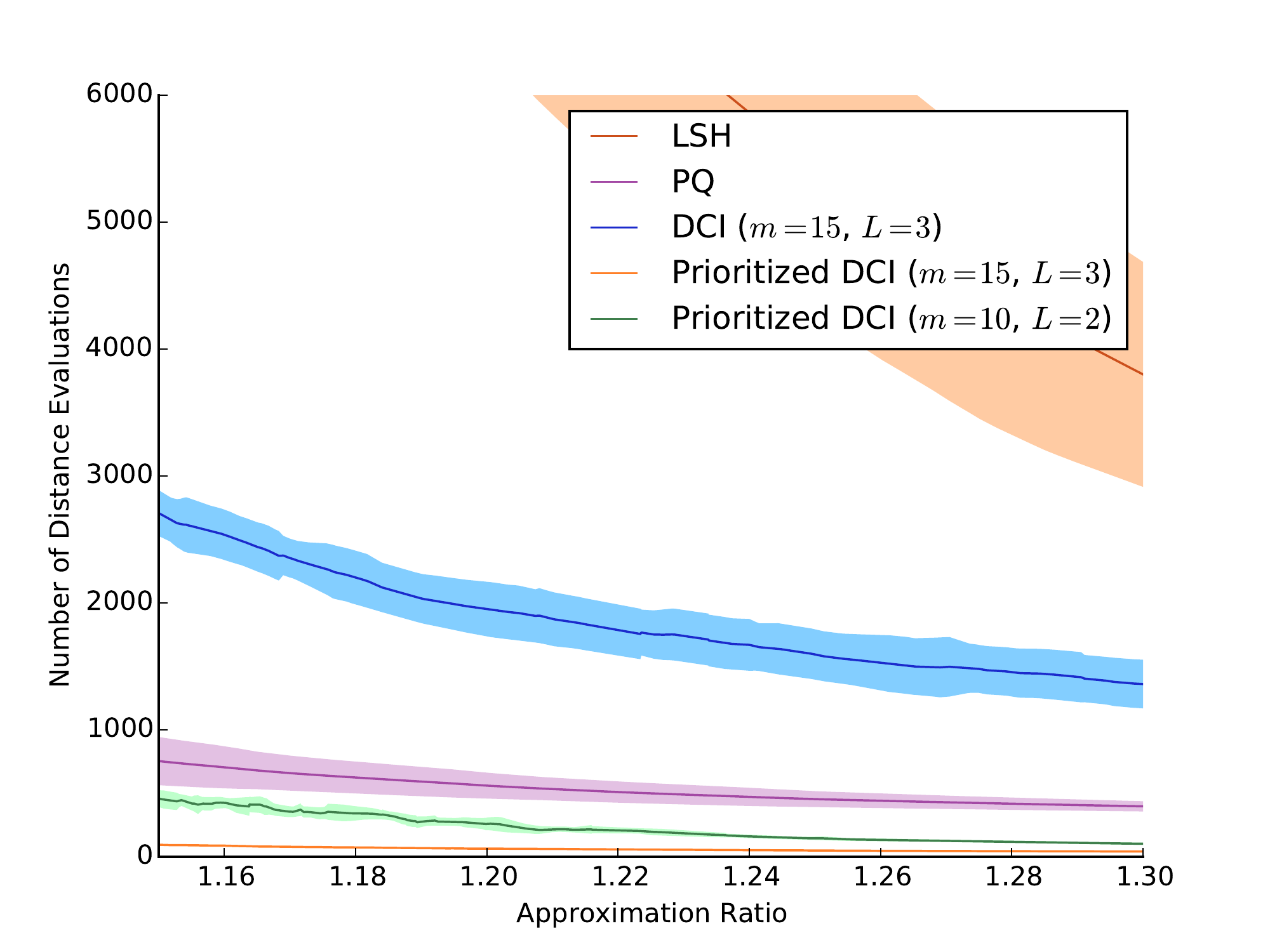}
        \label{fig:dist_evals_mnist_zoomed_in}
    }
    \caption{\label{fig:dist_evals}Comparison of the number of distance evaluations needed by different algorithms to achieve varying levels of approximation quality on (a) CIFAR-100 and (b,c) MNIST. Each curve represents the mean over ten folds and the shaded area represents $\pm$1 standard deviation. Lower values are better. (c) Close-up view of the figure in (b). }
\end{figure*}

We compare the performance of Prioritized DCI to that of standard DCI~\cite{Li2016FastKN}, product quantization~\cite{jegou2011product} and LSH~\cite{datar2004locality}, which is perhaps the algorithm that is most widely used in high-dimensional settings. Because LSH operates under the approximate setting, in which the performance metric of interest is how close the returned points are to the query rather than whether they are the true $k$-nearest neighbours. All algorithms are evaluated in terms of the time they would need to achieve varying levels of approximation quality. 

Evaluation is performed on two datasets, CIFAR-100~\cite{krizhevsky2009learning} and MNIST~\cite{lecun1998gradient}. CIFAR-100 consists of $60,000$ colour images of 100 types of objects in natural scenes and MNIST consists of $70,000$ grayscale images of handwritten digits. The images in CIFAR-100 have a size of $32 \times 32$ and three colour channels, and the images in MNIST have a size of $28 \times 28$ and a single colour channel. We reshape each image into a vector whose entries represent pixel intensities at different locations and colour channels in the image. So, each vector has a dimensionality of $32 \times 32 \times 3 = 3072$ for CIFAR-100 and $28 \times 28 = 784$ for MNIST. Note that the dimensionalities under consideration are much higher than those typically used to evaluate prior methods. 

For the purposes of nearest neighbour search, MNIST is a more challenging dataset than CIFAR-100. This is because images in MNIST are concentrated around a few modes; consequently, data points form dense clusters, leading to higher intrinsic dimensionality. On the other hand, images in CIFAR-100 are more diverse, and so data points are more dispersed in space. Intuitively, it is much harder to find the closest digit to a query among 6999 other digits of the same category that are all plausible near neighbours than to find the most similar natural image among a few other natural images with similar appearance. Later results show that all algorithms need fewer distance evaluations to achieve the same level of approximation quality on CIFAR-100 than on MNIST. 

We evaluate performance of all algorithms using cross-validation, where we randomly choose ten different splits of query vs. data points. Each split consists of 100 points from the dataset that serve as queries, with the remainder designated as data points. We use each algorithm to retrieve the 25 nearest neighbours at varying levels of approximation quality and report mean performance and standard deviation over all splits. 

Approximation quality is measured using the approximation ratio, which is defined to be the ratio of the radius of the ball containing the set of true $k$-nearest neighbours to the radius of the ball containing the set of approximate $k$-nearest neighbours returned by the algorithm. The closer the approximation ratio is to 1, the higher the approximation quality. In high dimensions, the time taken to compute true distances between the query and the candidate points dominates query time, so the number of distance evaluations can be used as an implementation-independent proxy for the query time. 

For LSH, we used $24$ hashes per table and $100$ tables, which we found to achieve the best approximation quality given the memory constraints. For product quantization, we used a data-independent codebook with $256$ entries so that the algorithm supports dynamic updates. For standard DCI, we used the same hyparameter settings used in \citep{Li2016FastKN} ($m = 25$ and $L = 2$ on CIFAR-100 and $m = 15$ and $L = 3$ on MNIST). For Prioritized DCI, we used two different settings: one that matches the hyperparameter settings of standard DCI, and another that uses less space ($m = 10$ and $L = 2$ on both CIFAR-100 and MNIST). 

We plot the number of distance evaluations that each algorithm requires to achieve each desired level of approximation ratio in Figure~\ref{fig:dist_evals}. As shown, on CIFAR-100, under the same hyperparameter setting used by standard DCI, Prioritized DCI requires $87.2\%$ to $92.5\%$ fewer distance evaluations than standard DCI, $91.7\%$ to $92.8\%$ fewer distance evaluations than product quantization, and $90.9\%$ to $93.8\%$ fewer distance evaluations than LSH to achieve same levels approximation quality, which represents a 14-fold reduction in the number of distance evaluations relative to LSH on average. Under the more space-efficient hyperparameter setting, Prioritized DCI achieves a 6-fold reduction compared to LSH. On MNIST, under the same hyperparameter setting used by standard DCI, Prioritized DCI requires $96.4\%$ to $97.0\%$ fewer distance evaluations than standard DCI, $87.1\%$ to $89.8\%$ fewer distance evaluations than product quantization, and $98.8\%$ to $99.3\%$ fewer distance evaluations than LSH, which represents a 116-fold reduction relative to LSH on average. Under the more space-efficient hyperparameter setting, Prioritized DCI achieves a 32-fold reduction compared to LSH. 

We compare the space efficiency of Prioritized DCI to that of standard DCI and LSH. As shown in Figure~\ref{fig:mem_comp} in the supplementary material, compared to LSH, Prioritized DCI uses $95.5\%$ less space on CIFAR-100 and $95.3\%$ less space on MNIST under the same hyperparameter settings used by standard DCI. This represents a 22-fold reduction in memory consumption on CIFAR-100 and a 21-fold reduction on MNIST. Under the more space-efficient hyperparameter setting, Prioritized DCI uses $98.2\%$ less space on CIFAR-100 and $97.9\%$ less space on MNIST relative to LSH, which represents a 55-fold reduction on CIFAR-100 and a 48-fold reduction on MNIST. 

In terms of wall-clock time, our implementation of Prioritized DCI takes $1.18$ seconds to construct the data structure and execute 100 queries on MNIST, compared to $104.71$ seconds taken by LSH. 

\section{Conclusion}

In this paper, we presented a new exact randomized algorithm for $k$-nearest neighbour search, which we refer to as Prioritized DCI. We showed that Prioritized DCI achieves a significant improvement in terms of the dependence of query time complexity on intrinsic dimensionality compared to standard DCI. Specifically, Prioritized DCI can to a large extent counteract a \emph{linear} increase in the intrinsic dimensionality, or equivalently, an \emph{exponential} increase in the number of points near a query, using just a \emph{linear} increase in the number of simple indices. Empirical results validated the effectiveness of Prioritized DCI in practice, demonstrating the advantages of Prioritized DCI over prior methods in terms of speed and memory usage. 

\paragraph{Acknowledgements.} This work was supported by DARPA W911NF-16-1-0552. Ke Li thanks the Natural Sciences and Engineering Research Council of Canada (NSERC) for fellowship support. 
\bibliography{pdci}
\bibliographystyle{icml2017}

\newpage

\twocolumn[
\icmltitle{Fast $k$-Nearest Neighbour Search via Prioritized DCI \\ \vspace{5pt} \large Supplementary Material}

\icmlsetsymbol{equal}{*}

\begin{icmlauthorlist}
\icmlauthor{Ke Li}{berkeley}
\icmlauthor{Jitendra Malik}{berkeley}
\end{icmlauthorlist}

\icmlaffiliation{berkeley}{University of California, Berkeley, CA 94720, United States}

\icmlcorrespondingauthor{Ke Li}{ke.li@eecs.berkeley.edu}

\vskip 0.3in
]

\section{Analysis}

We present proofs that were omitted from the main paper below. 

\begin{customthm}{1}
Let $\left\{ v^{l}_{i}\right\} _{i=1}^{N}$ and $\left\{ v^{s}_{i'} \vphantom{v^{l}_{i}} \right\} _{i'=1}^{N'}$ be sets of vectors such that $\left\Vert v^{l}_{i}\right\Vert _{2} > \left\Vert \vphantom{v^{l}} v^{s}_{i'}\right\Vert _{2}\;\forall i\in[N],i'\in[N']$. Furthermore, let $\left\{ u_{ij}'\right\} _{i\in[N],j\in[M]}$ be random uniformly distributed unit vectors such that $u_{i1}',\ldots,u_{iM}'$ are independent for any given $i$. Consider the events $\left\{ \exists v_{i'}^{s}\mbox{ s.t. }\max_{j}\left\{ \left|\langle v_{i}^{l},u_{ij}'\rangle\right|\right\} \leq\left\Vert v_{i'}^{s}\vphantom{v^{l}}\right\Vert _{2}\right\} _{i=1}^{N}$. The probability that at least $k'$ of these events occur is at most $\frac{1}{k'}\sum_{i=1}^{N}\left(1-\frac{2}{\pi}\cos^{-1}\left(\left\Vert v^{s}_{\mathrm{max}}\vphantom{v^{l}}\right\Vert _{2}/\left\Vert v_{i}^{l}\right\Vert _{2}\right)\right)^{M}$, where $\left\Vert v^{s}_{\mathrm{max}}\vphantom{v^{l}}\right\Vert _{2} = \max_{i'}\left\{ \left\Vert v_{i'}^{s}\vphantom{v^{l}}\right\Vert _{2}\right\} $. Furthermore, if $k' = N$, it is at most $\min_{i\in[N]}\left\{ \left(1-\frac{2}{\pi}\cos^{-1}\left(\left\Vert v^{s}_{\mathrm{max}}\vphantom{v^{l}}\right\Vert _{2}/\left\Vert v_{i}^{l}\right\Vert _{2}\right)\right)^{M} \right\}$. 
\end{customthm}

\begin{proof}
The event that $\exists v_{i'}^{s}\mbox{ s.t. }\max_{j}\left\{ \left|\langle v_{i}^{l},u_{ij}'\rangle\right|\right\} \leq\left\Vert v_{i'}^{s}\vphantom{v^{l}}\right\Vert _{2}$ is equivalent to the event that $\max_{j}\left\{ \left|\langle v_{i}^{l},u_{ij}'\rangle\right|\right\} \leq\max_{i'}\left\{ \left\Vert v_{i'}^{s}\vphantom{v^{l}}\right\Vert _{2}\right\} =\left\Vert v_{\mathrm{max}}^{s}\vphantom{v^{l}}\right\Vert _{2}$. Take $E_{i}$ to be the event that $\max_{j}\left\{ \left|\langle v_{i}^{l},u_{ij}'\rangle\right|\right\} \leq\left\Vert v_{\mathrm{max}}^{s}\vphantom{v^{l}}\right\Vert _{2}$. By Lemma~\ref{lem:order_invert}, $\mathrm{Pr}(E_{i})\leq\left(1-\frac{2}{\pi}\cos^{-1}\left(\left\Vert v^{s}_{\mathrm{max}}\vphantom{v^{l}}\right\Vert _{2}/\left\Vert v_{i}^{l}\right\Vert _{2}\right)\right)^{M}$. It follows from Lemma~\ref{lem:choose_event} that the probability that $k'$ of $E_{i}$'s occur is at most $\frac{1}{k'}\sum_{i=1}^{N}\mathrm{Pr}\left(E_{i}\right)\leq\frac{1}{k'}\sum_{i=1}^{N}\left(1-\frac{2}{\pi}\cos^{-1}\left(\left\Vert v^{s}_{\mathrm{max}}\vphantom{v^{l}}\right\Vert _{2}/\left\Vert v_{i}^{l}\right\Vert _{2}\right)\right)^{M}$. If $k'=N$, we use the fact that $\bigcap_{i'=1}^{N}E_{i'}\subseteq E_{i}\;\forall i$, which implies that $\mathrm{Pr}\left(\bigcap_{i'=1}^{N}E_{i'}\right)\leq\min_{i\in[N]}\mathrm{Pr}\left(E_{i}\right)\leq\min_{i\in[N]}\left\{ \left(1-\frac{2}{\pi}\cos^{-1}\left(\left\Vert v^{s}_{\mathrm{max}}\vphantom{v^{l}}\right\Vert _{2}/\left\Vert v_{i}^{l}\right\Vert _{2}\right)\right)^{M}\right\}$.  
\end{proof}

\begin{customlem}{3}
Consider points in the order they are retrieved from a composite index that consists of $m$ simple indices. The probability that there are at least $n_{0}$ points that are not the true $k$-nearest neighbours but are retrieved before some of them is at most $\frac{1}{n_{0}-k}\sum_{i=2k+1}^{n}\left(1-\frac{2}{\pi}\cos^{-1}\left(\left\Vert p^{(k)}-q\right\Vert _{2}/\left\Vert p^{(i)}-q\right\Vert _{2}\right)\right)^{m}$. 
\end{customlem}

\begin{proof}
Points that are not the true $k$-nearest neighbours but are retrieved before some of them will be referred to as \emph{extraneous points} and are divided into two categories: \emph{reasonable} and \emph{silly}. An extraneous point is reasonable if it is one of the $2k$-nearest neighbours, and is silly otherwise. For there to be $n_{0}$ extraneous points, there must be $n_{0} - k$ silly extraneous points. Therefore, the probability that there are $n_{0}$ extraneous points is upper bounded by the probability that there are $n_{0} - k$ silly extraneous points. 

Since points are retrieved from the composite index in the order of increasing maximum projected distance to the query, for any pair of points $p$ and $p'$, if $p$ is retrieved before $p'$, then $\max_{j}\left\{ \left|\langle p-q,u_{jl}\rangle\right|\right\} \leq\max_{j}\left\{ \left|\langle p'-q,u_{jl}\rangle\right|\right\}$, where $\left\{ u_{jl}\right\} _{j=1}^{m}$ are the projection directions associated with the constituent simple indices of the composite index. 

By Theorem~\ref{thm:multi_order_invert}, if we take $\left\{ v_{i}^{l}\right\} _{i=1}^{N}$ to be $\left\{ p^{(i)}-q\right\} _{i=2k+1}^{n}$, $\left\{ v_{i'}^{s}\vphantom{v_{i}^{l}}\right\} _{i'=1}^{N'}$ to be $\left\{ p^{(i)}-q\right\} _{i=1}^{k}$, $M$ to be $m$, $\left\{u_{ij}'\right\} _{j\in[M]}$ to be $\left\{u_{jl}\right\} _{j\in[m]}$ for all $i \in [N]$ and $k'$ to be $n_{0}-k$, we obtain an upper bound for the probability of there being a subset of $\left\{ p^{(i)}\right\} _{i=2k+1}^{n}$ of size $n_{0}-k$ such that for all points $p$ in the subset, $\max_{j}\left\{ \left|\langle p-q,u_{jl}\rangle\right|\right\} \leq\left\Vert p'-q\right\Vert _{2}$ for some $p' \in \left\{ p^{(i)}-q\right\} _{i=1}^{k}$. In other words, this is the probability of there being $n_{0}-k$ points that are not the $2k$-nearest neighbours whose maximum projected distances are no greater than the distance from some $k$-nearest neighbours to the query, which is at most $\frac{1}{n_{0}-k}\sum_{i=2k+1}^{n}\left(1-\frac{2}{\pi}\cos^{-1}\left(\left\Vert p^{(k)}-q\right\Vert _{2}/\left\Vert p^{(i)}-q\right\Vert _{2}\right)\right)^{m}$. 

Since the event that $\max_{j}\left\{ \left|\langle p-q,u_{jl}\rangle\right|\right\} \leq\max_{j}\left\{ \left|\langle p'-q,u_{jl}\rangle\right|\right\}$ is contained in the event that $\max_{j}\left\{ \left|\langle p-q,u_{jl}\rangle\right|\right\} \leq\left\Vert p'-q\right\Vert _{2}$ for any $p,p'$, this is also an upper bound for the probability of there being $n_{0}-k$ points that are not the $2k$-nearest neighbours whose maximum projected distances do not exceed those of some of the $k$-nearest neighbours, which by definition is the probability that there are $n_{0} - k$ silly extraneous points. Since this probability is no less than the probability that there are $n_{0}$ extraneous points, the upper bound also applies to this probability. 
\end{proof}

\begin{customlem}{4}
Consider point projections in a composite index that consists of $m$ simple indices in the order they are visited. The probability that $n_{0}$ point projections that are not of the true $k$-nearest neighbours are visited before all true $k$-nearest neighbours have been retrieved is at most $\frac{m}{n_{0}-mk}\sum_{i=2k+1}^{n}\left(1-\frac{2}{\pi}\cos^{-1}\left(\left\Vert p^{(k)}-q\right\Vert _{2}/\left\Vert p^{(i)}-q\right\Vert _{2}\right)\right)$. 
\end{customlem}

\begin{proof}
Projections of points that are not the true $k$-nearest neighbours but are visited before the $k$-nearest neighbours have all been retrieved will be referred to as \emph{extraneous projections} and are divided into two categories: \emph{reasonable} and \emph{silly}. An extraneous projection is reasonable if it is of one of the $2k$-nearest neighbours, and is silly otherwise. For there to be $n_{0}$ extraneous projections, there must be $n_{0} - mk$ silly extraneous projections, since there could be at most $mk$ reasonable extraneous projections. Therefore, the probability that there are $n_{0}$ extraneous projections is upper bounded by the probability that there are $n_{0} - mk$ silly extraneous projections. 

Since point projections are visited in the order of increasing projected distance to the query, each extraneous silly projection must be closer to the query projection than the maximum projection of some $k$-nearest neighbour. 

By Theorem~\ref{thm:multi_order_invert}, if we take $\left\{ v_{i}^{l}\right\} _{i=1}^{N}$ to be $\left\{ p^{(2k+\lfloor (i-1)/m\rfloor +1)}-q\right\} _{i=1}^{m(n-2k)}$, $\left\{ v_{i'}^{s}\vphantom{v_{i}^{l}}\right\} _{i'=1}^{N'}$ to be $\left\{ p^{(\lfloor(i-1)/m\rfloor +1)}-q\right\} _{i=1}^{mk}$, $M $ to be $1$, $\left\{u_{i1}'\right\}_{i=1}^{N}$ to be $\left\{u_{(i \mod m),l}\right\} _{i=1}^{m(n-2k)}$ and $k'$ to be $n_{0}-mk$, we obtain an upper bound for the probability of there being $n_{0}-mk$ point projections that are not of the $2k$-nearest neighbours whose distances to their respective query projections are no greater than the true distance between the query and some $k$-nearest neighbour, which is $\frac{1}{n_{0}-mk}\sum_{i=2k+1}^{n}m\left(1-\frac{2}{\pi}\cos^{-1}\left(\frac{\left\Vert p^{(k)}-q\right\Vert _{2}}{\left\Vert p^{(i)}-q\right\Vert _{2}}\right)\right)$. 

Because maximum projected distances are no more than true distances, this is also an upper bound for the probability of there being $n_{0} - mk$ silly extraneous projections. Since this probability is no less than the probability that there are $n_{0}$ extraneous projections, the upper bound also applies to this probability. 
\end{proof}

\begin{customlem}{5}
On a dataset with global relative sparsity $(k,\gamma)$, the quantity $\sum_{i=2k+1}^{n}\left(1-\frac{2}{\pi}\cos^{-1}\left(\left\Vert p^{(k)}-q\right\Vert _{2}/\left\Vert p^{(i)}-q\right\Vert _{2}\right)\right)^{m}$ is at most $O\left(k\max(\log(n/k),(n/k)^{1-m\log_{2}\gamma})\right)$. 
\end{customlem}

\begin{proof}
By definition of global relative sparsity, for all $i\geq2k+1$, $\left\Vert p^{(i)}-q\right\Vert _{2}>\gamma\left\Vert p^{(k)}-q\right\Vert _{2}$. A recursive application shows that for all $i\geq2^{i'}k+1$, $\left\Vert p^{(i)}-q\right\Vert _{2}>\gamma^{i'}\left\Vert p^{(k)}-q\right\Vert _{2}$. 

Applying the fact that $1-(2 / \pi)\cos^{-1}\left(x\right)\leq x \; \forall x \in [0,1]$ and the above observation yields:

\vspace{-10pt}
\footnotesize
\begin{align*}
 & \sum_{i=2k+1}^{n}\left(1-\frac{2}{\pi}\cos^{-1}\left(\frac{\left\Vert p^{(k)}-q\right\Vert _{2}}{\left\Vert \vphantom{p^{(k)}} p^{(i)}-q\right\Vert _{2}}\right)\right)^{m}
\end{align*}
\begin{align*}
 \leq \;& \sum_{i=2k+1}^{n}\left(\frac{\left\Vert p^{(k)}-q\right\Vert _{2}}{\left\Vert \vphantom{p^{(k)}} p^{(i)}-q\right\Vert _{2}}\right)^{m} \\
 < \;& \sum_{i'=1}^{\lceil\log_{2}(n/k)\rceil-1}2^{i'}k\gamma^{-i'm}
\end{align*}
\normalsize

If $\gamma \geq \sqrt[m]{2}$, this quantity is at most $k\log_{2}\left(n/k\right)$. On the other hand, if $1 \leq\gamma < \sqrt[m]{2}$, this quantity can be simplified to:
\vspace{-10pt}
\footnotesize
\begin{align*}
 & k\left(\frac{2}{\gamma^{m}}\right)\left(\left(\frac{2}{\gamma^{m}}\right)^{\lceil\log_{2}(n/k)\rceil-1}-1\right)/\left(\frac{2}{\gamma^{m}}-1\right) \\
 = \;& O\left(k\left(\frac{2}{\gamma^{m}}\right)^{\lceil\log_{2}(n/k)\rceil-1}\right) \\
 = \;& O\left(k\left(\frac{n}{k}\right)^{1-m\log_{2}\gamma}\right)
\end{align*}
\normalsize
Therefore, $\sum_{i=2k+1}^{n}\left(\left\Vert p^{(k)}-q\right\Vert _{2}/\left\Vert p^{(i)}-q\right\Vert _{2}\right)^{m} \leq O\left(k\max(\log(n/k),(n/k)^{1-m\log_{2}\gamma})\right)$. 
\end{proof}

\begin{customlem}{6}
For a dataset with global relative sparsity $(k,\gamma)$ and a given composite index consisting of $m$ simple indices, there is some $k_{0} \in\Omega(k\max(\log(n/k),(n/k)^{1-m\log_{2}\gamma}))$ such that the probability that the candidate points retrieved from the composite index do not include some of the true $k$-nearest neighbours is at most some constant $\alpha_{0}<1$.
\end{customlem}

\begin{proof}
We will refer to the true $k$-nearest neighbours that are among first $k_{0}$ points retrieved from the composite index as \emph{true positives} and those that are not as \emph{false negatives}. Additionally, we will refer to points that are not true $k$-nearest neighbours but are among the first $k_{0}$ points retrieved as \emph{false positives}. 

When not all the true $k$-nearest neighbours are among the first $k_{0}$ candidate points, there must be at least one false negative and so there can be at most $k-1$ true positives. Consequently, there must be at least $k_{0}-(k-1)$ false positives. To find an upper bound on the probability of the existence of $k_{0}-(k-1)$ false positives in terms of global relative sparsity, we apply Lemma \ref{lem:prob_num_extraneous_retrieved_points} with $n_{0}$ set to $k_{0}-(k-1)$, followed by Lemma~\ref{lem:sum_dist_ratios}. We conclude that this probability is at most $\frac{1}{k_{0}-2k+1}O\left(k\max(\log(n/k),(n/k)^{1-m\log_{2}\gamma})\right)$. Because the event that not all the true $k$-nearest neighbours are among the first $k_{0}$ candidate points is contained in the event that there are $k_{0}-(k-1)$ false positives, the former is upper bounded by the same quantity. So, we can choose some $k_{0} \in \Omega(k\max(\log (n/k),(n/k)^{1-m\log_{2}\gamma}))$ to make it strictly less than 1. 
\end{proof}

\begin{customlem}{7}
For a dataset with global relative sparsity $(k,\gamma)$ and a given composite index consisting of $m$ simple indices, there is some $k_{1} \in\Omega(mk\max(\log(n/k),(n/k)^{1-\log_{2}\gamma}))$ such that the probability that the candidate points retrieved from the composite index do not include some of the true $k$-nearest neighbours is at most some constant $\alpha_{1}<1$.
\end{customlem}

\begin{proof}
We will refer to the projections of true $k$-nearest neighbours that are among first $k_{1}$ visited point projections as \emph{true positives} and those that are not as \emph{false negatives}. Additionally, we will refer to projections of points that are not of the true $k$-nearest neighbours but are among the first $k_{1}$ visited point projections as \emph{false positives}. 

When a $k$-nearest neighbour is not among the candidate points that have been retrieved, some of its projections must not be among the first $k_{1}$ visited point projections. So, there must be at least one false negative, implying that there can be at most $mk-1$ true positives. Consequently, there must be at least $k_{1}-(mk-1)$ false positives. To find an upper bound on the probability of the existence of $k_{1}-(mk-1)$ false positives in terms of global relative sparsity, we apply Lemma \ref{lem:prob_num_extraneous_visited_points} with $n_{0}$ set to $k_{1}-(mk-1)$, followed by Lemma~\ref{lem:sum_dist_ratios}. We conclude that this probability is at most $\frac{m}{k_{1}-2mk+1}O\left(k\max(\log(n/k),(n/k)^{1-\log_{2}\gamma})\right)$. Because the event that some true $k$-nearest neighbour is missing from the candidate points is contained in the event that there are $k_{1}-(mk-1)$ false positives, the former is upper bounded by the same quantity. So, we can choose some $k_{1} \in \Omega(mk\max(\log (n/k),(n/k)^{1-\log_{2}\gamma}))$ to make it strictly less than 1. 
\end{proof}

\begin{customthm}{2}
For a dataset with global relative sparsity $(k,\gamma)$, for any $\epsilon > 0$, there is some $L$, $k_{0} \in \Omega(k\max(\log(n/k),(n/k)^{1-m\log_{2}\gamma}))$ and $k_{1} \in \Omega(mk\max(\log(n/k),(n/k)^{1-\log_{2}\gamma}))$ such that the algorithm returns the correct set of $k$-nearest neighbours with probability of at least $1 - \epsilon$. 
\end{customthm}

\begin{proof}
For a given composite index, by Lemma~\ref{lem:single_k0_failure_prob}, there is some $k_{0} \in \Omega(k\max(\log(n/k),(n/k)^{1-m\log_{2}\gamma}))$ such that the probability that some of the true $k$-nearest neighbours are missed is at most some constant $\alpha_{0}<1$. Likewise, by Lemma~\ref{lem:single_k1_failure_prob}, there is some $k_{1} \in \Omega(mk\max(\log(n/k),(n/k)^{1-\log_{2}\gamma}))$ such that this probability is at most some constant $\alpha_{1}<1$. By choosing such $k_{0}$ and $k_{1}$, this probability is therefore at most $\min\{\alpha_{0},\alpha_{1}\}<1$. For the algorithm to fail, all composite indices must miss some $k$-nearest neighbours. Since each composite index is constructed independently, the algorithm fails with probability of at most $\left(\min\{\alpha_{0},\alpha_{1}\}\right)^{L}$, and so must succeed with probability of at least $1-\left(\min\{\alpha_{0},\alpha_{1}\}\right)^{L}$. Since $\min\{\alpha_{0},\alpha_{1}\} < 1$, there is some $L$ that makes $1-\left(\min\{\alpha_{0},\alpha_{1}\}\right)^{L} \geq 1 - \epsilon$. 
\end{proof}

\begin{customthm}{3}
For a given number of simple indices $m$, the algorithm takes $O\left(dk\max(\log(n/k),(n/k)^{1-m/d'})+\right.$ $\left.mk\log m\left(\max(\log(n/k),(n/k)^{1-1/d'})\right)\right)$ time to retrieve the $k$-nearest neighbours at query time, where $d'$ denotes the intrinsic dimensionality. 
\end{customthm}

\begin{proof}
Computing projections of the query point along all $u_{jl}$'s takes $O(dm)$ time, since $L$ is a constant. Searching in the binary search trees/skip lists $T_{jl}$'s takes $O(m \log n)$ time. The total number of point projections that are visited is at most $\Theta(mk\max(\log(n/k),(n/k)^{1-\log_{2}\gamma}))$. Because determining the next point to visit requires popping and pushing a priority queue, which takes $O(\log m)$ time, the total time spent on visiting points is $O(mk \log m\max(\log(n/k),(n/k)^{1-\log_{2}\gamma}))$. The total number of candidate points retrieved is at most $\Theta(k\max(\log (n/k),(n/k)^{1-m\log_{2}\gamma}))$. Because true distances are computed for every candidate point, the total time spent on distance computation is $O(dk\max(\log (n/k),(n/k)^{1-m\log_{2}\gamma}))$. We can find the $k$ closest points to the query among the candidate points using a selection algorithm like quickselect, which takes $O(k\max(\log (n/k),(n/k)^{1-m\log_{2}\gamma}))$ time on average. Since the time for visiting points and for computing distances dominates, the entire algorithm takes $O(dk\max(\log (n/k),(n/k)^{1-m\log_{2}\gamma}) + mk \log m\max(\log(n/k),(n/k)^{1-\log_{2}\gamma}))$ time. Substituting $1/d'$ for $\log_{2}\gamma$ yields the desired expression. 
\end{proof}

\begin{figure*}[t]
    \centering
    \subfloat[]{
        \includegraphics[width=0.5\textwidth]{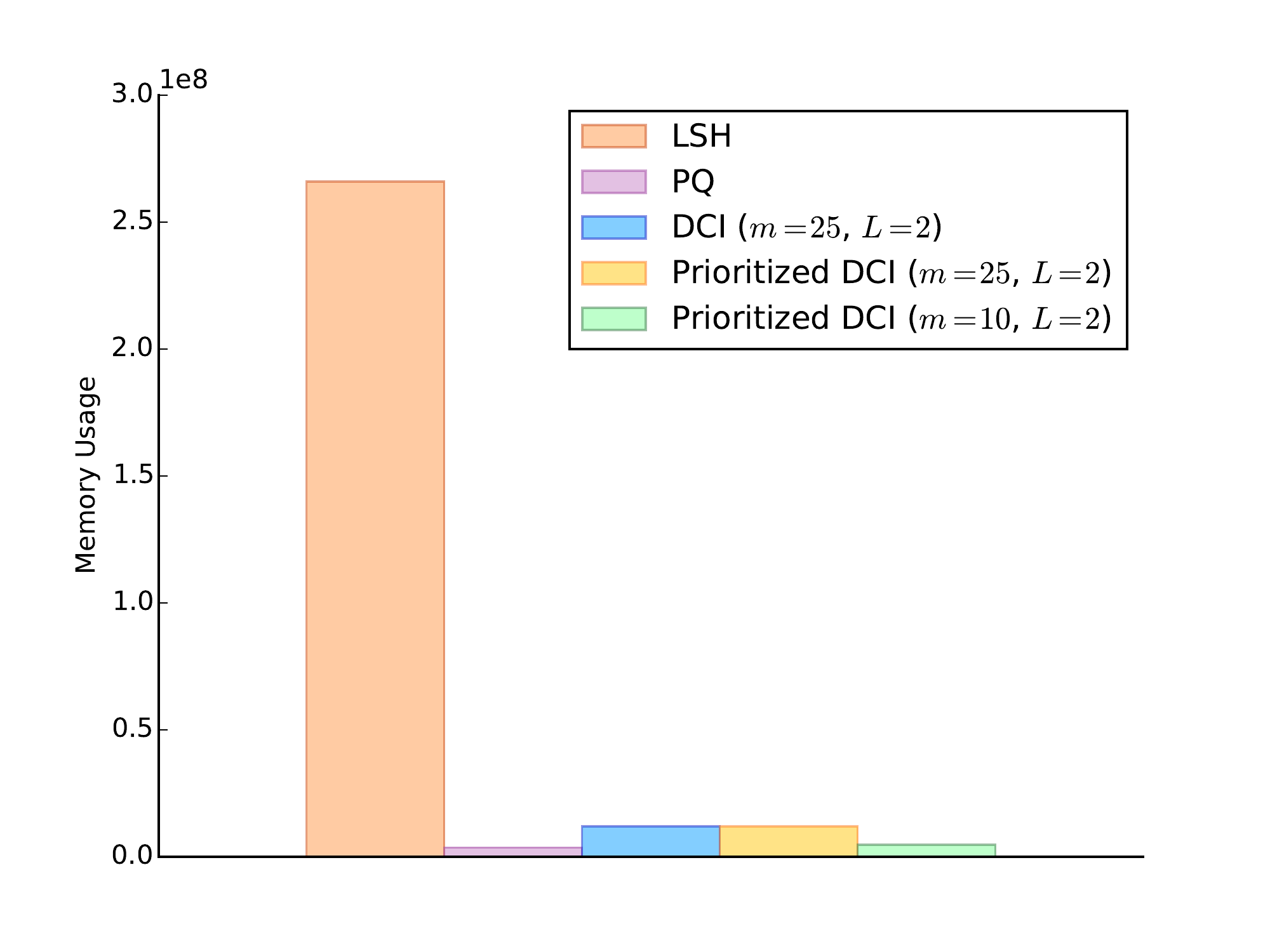}
        \label{fig:mem_comp_cifar}
    }
    \subfloat[]{
        \includegraphics[width=0.5\textwidth]{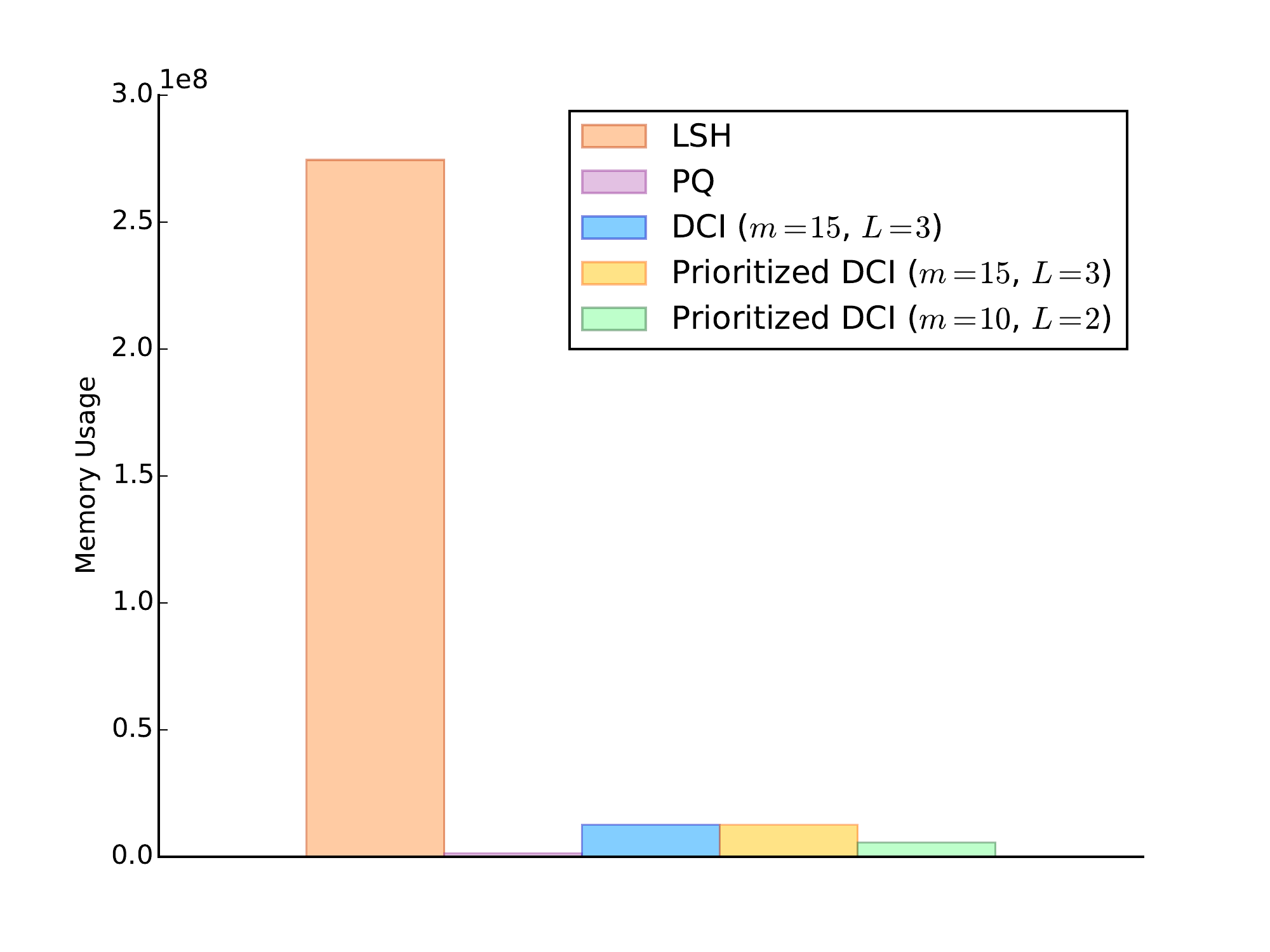}
        \label{fig:mem_comp_mnist}
    }
    \caption{\label{fig:mem_comp}Memory usage of different algorithms on (a) CIFAR-100 and (b) MNIST. Lower values are better. }
\end{figure*}

\begin{customthm}{4}
For a given number of simple indices $m$, the algorithm takes $O(m(dn+n\log n))$ time to preprocess the data points in $D$ at construction time. 
\end{customthm}

\begin{proof}
Computing projections of all $n$ points along all $u_{jl}$'s takes $O(dmn)$ time, since $L$ is a constant. Inserting all $n$ points into $mL$ self-balancing binary search trees/skip lists takes $O(mn\log n)$ time. 
\end{proof}

\begin{customthm}{5}
The algorithm requires $O(m(d+\log n))$ time to insert a new data point and $O(m \log n)$ time to delete a data point. 
\end{customthm}

\begin{proof}
In order to insert a data point, we need to compute its projection along all $u_{jl}$'s and insert it into each binary search tree or skip list. Computing the projections takes $O(md)$ time and inserting them into the corresponding self-balancing binary search trees or skip lists takes $O(m \log n)$ time. In order to delete a data point, we simply remove its projections from each of the binary search trees or skip lists, which takes $O(m \log n)$ time. 
\end{proof}

\begin{customthm}{6}
The algorithm requires $O(mn)$ space in addition to the space used to store the data. 
\end{customthm}

\begin{proof}
The only additional information that needs to be stored are the $mL$ binary search trees or skip lists. Since $n$ entries are stored in each binary search tree/skip list, the total additional space required is $O(mn)$. 
\end{proof}

\section{Experiments}

Figure~\ref{fig:mem_comp} shows the memory usage of different algorithms on CIFAR-100 and MNIST.

\end{document}